\documentclass[journal,twoside]{ieeeconf}
\usepackage{cite}
\usepackage{amsmath,amssymb,amsfonts}
\usepackage[top=60pt, left=50pt, right=50pt, bottom=50pt]{geometry}
\usepackage{graphicx}
\usepackage{textcomp}
\usepackage{amstext}
\usepackage{algorithm}
\usepackage[noend]{algpseudocode}
\usepackage{bm}

\newcommand{\done}{\hspace*{\fill} $\Box$}
\newcommand{\RE}{I\!\!R}
\newtheorem{theorem}{Theorem}[section]

\newtheorem{lemma}[theorem]{Lemma}
\newtheorem{remark}[theorem]{Remark}

\begin{document}
\title{Regularized Least Squares Training of Quadratic Neural Networks with Applications to System Identification}
%\thanks{Grants or other notes
%about the article that should go on the front page should be

\author{%\vspace{0.5cm}
L. Rodrigues, Z. Yetman Van Egmond, \and M. R. Amiri Fard
 \\
\small Department of Electrical and Computer Engineering \\
\small Concordia University \\
\small Montr\'eal, QC, Canada\\
\small Contact Email: {\tt luis.rodrigues@concordia.ca}
}

\onecolumn

\maketitle

\begin{abstract}
This paper proposes a least squares approach for the training of quadratic neural networks with regularization.
The proposed methodology yields a lower bound on the solution of the training optimization problem for the case where the regularization coefficient is positive.
Moreover, it yields closed-form expressions for the approximate solution and its sensitivity
The lower bound is tight and the approximate solution is the optimal solution when the regularization coefficient is zero.
Having a closed-form expression for the weights reduces considerably the computational time when compared with iterative numerical methods such as backpropagation that can get stuck in local minima.
The proposed approach has three main contributions, namely, (i) it yields an analytical expression for the weights, (ii) an analytical expression for the sensitivity of the weights to errors in the data is also provided, (iii) it establishes a connection between the optimization to compute a lower bound and nuclear norm minimization.
The proposed least squares training is successfully applied to a nonlinear system identification example where the proposed lower bound is compared with the optimal value.
%\keywords{quadratic neural networks \and semidefinite programmning \and regularized least squares \and system identification}
% \PACS{PACS code1 \and PACS code2 \and more}
%\subclass{90C22 \and 93E24}
\end{abstract}

\section{Introduction}
Artificial neural networks have been an active area of research with several applications that are pervasive in our lives.
The training of a neural network is generally formulated as a nonconvex optimization problem and locally optimal solutions are found using backpropagation \cite{rumelhart1986backprop}.
Unfortunately, backpropagation does not guarantee that the globally optimal weights of the network are found.
This has spurred interest in convex formulations of the training of neural networks.
One of the first papers to address convex neural networks in a general setup was reference \cite{Bengio2005}.
Our focus will be on convex training for a specific type of neural networks with quadratic activation functions.
A growing interest on quadratic neural networks with different forms of activation has been
clearly summarized in table 1 of reference \cite{Xu2022QuadraLibAP}.
More recently, the training of two-layer quadratic neural networks (QNNs) was formulated in reference \cite{BartanPilanci2021} as a convex optimization program, which can be efficiently solved to a global optimum point.
Two-layer quadratic neural networks with a single hidden layer also offer other advantages.
In particular, the input and output are related by a quadratic form and the architecture of the network is a by-product of the training itself.
This comes at the expense of only having one hidden layer.
Extensions to more hidden layers have been proposed in \cite{Bartandeep2022} and more recently in \cite{Medpaper2026}.
In particular, it is shown in reference \cite{Medpaper2026} that the training of a deep QNN is equivalent to the training of a two-layer QNN with monomial lifting of the input.
This motivates the research on efficient (convex) optimization formulations for the training of two-layer QNNs with closed-form solutions.
QNNs have been successfully designed for several applications, including system identification and control of dynamical systems with Lyapunov stability guarantees \cite{RodriguesGivigi2023}.

Training of QNNs using both a $2$-norm loss function  and a $1$-norm regularization term in the cost function has been proposed and reformulated as a convex optimization in reference \cite{BartanPilanci2021} but without a closed-form solution for the weights.
Preliminary work of the first author \cite{Rodrigues2023} has shown that the optimal solution of the training of QNNs with a  a $2$-norm loss function but without regularization can be found by standard least squares, which yields a closed-form expression for the optimal weights.
The present paper is an extension of the work in  \cite{Rodrigues2023} to the more general case of QNN training using not only a $2$-norm loss function  but also including a $1$-norm regularization term in the cost function.
More specifically, the main contributions of this paper are:
\begin{enumerate}
\item a lower bound on the training cost for two-layer QNNs that can be obtained using regularized least squares,
\item analytical expressions for both the weights and the sensitivity of the weights,
%\item a significant reduction of the computational time for training the neural network when compared to backpropagation, which is extremely important when considering real-time applications,
\item a connection between the optimization to compute a lower bound and nuclear norm minimization.
\end{enumerate}
To the best of our knowledge this is the first least squares training approach that can be used for deep neural networks 
(using lifting of the inputs \cite{Medpaper2026}) that does not
need to fix or randomize the weights in the hidden layers and can treat them as decision variables.
The paper is structured as follows.
Section \ref{quadraticnetworks} reviews previous work and formulates the QNN training problem.
Bounds on the value of the training cost function are established in section \ref{trainingbounds}, followed
by the least squares training of QNNs in section \ref{leastsquares}.
The proposed least squares solution is then successfully applied to a nonlinear system identification example in section \ref{sysidexample} followed by conclusions.

\section{Quadratic Neural Networks (QNNs)}\label{quadraticnetworks}
Deep QNNs have been shown to be equivalent to a two-layer QNN with
the addition of extra inputs corresponding to the monomial lifting of the original input \cite{Medpaper2026}.
Therefore, having efficient methods to train two-layer QNNs also allows efficient training of deep QNNs.
A two-layer quadratic feedforward neural network has a single hidden layer with $M=\sum_{k=1}^pM_k$ neurons and the outputs are given by
\begin{equation}\label{thenetworkoutput}
\hat y^k(x) = \hat f^k(x) = \sum_{j=1}^{M_k}\sigma\left(x^Tw^{j}\right)\alpha_j^k
\end{equation}
for $k=1,\ldots,p,$ where the activation function is quadratic and is written as
\begin{equation}\label{activationfunction}
\sigma(z) = az^2 + bz + c
\end{equation}
where $a\neq 0, b, c,$ are pre-defined constants that parameterize the quadratic activation function.
A typical choice for these parameters is obtained by performing a least squares approximation of the ReLU 
function in a given interval, as suggested in \cite{BartanPilanci2021}.
The notation used in this paper for network connecting weights consists of subscripts to denote the source and superscripts
to denote the destination.
For example, the weights $w^{j}=[w^j_1\ldots w^j_n]^T$ connect the (source) input $x=[x_1\ldots x_n]\in\RE^n$ to each neuron $j$ in the hidden layer, whereas the weights $\alpha_j^k$ connect each neuron $j$ to the output $k$.
\iffalse
The notation of equation (\ref{networkoutput}) where the weights leaving the input are denoted by $w$ and the weights leaving the hidden layer are denoted by $\alpha$ will be used throughout the paper. For the weights we will use superscript indices to indicate the target  and subscript indices to indicate the source of each connection.
For example, to denote the weight connecting the first input component to the second hidden neuron of the third output we use $w_1^{3,2}$.
As another example, to denote the weight that connects the output of the second neuron in the hidden layer with the fifth output of the network we use $\alpha_2^5$. 
\fi
The desired (label) outputs will be denoted by $y=[y^1~\ldots~y^p]^T$ and the actual outputs of the network will be denoted by  $\hat y=[\hat y^1~\ldots~\hat y^p]^T$.
Since the output is a destination we use superscripts for its coordinates.
\iffalse
One of the advantages of quadratic neural networks is that the training of the weights can be done by solving a convex optimization problem, which guarantees convergence to the global optimum.
\fi
Using a convex loss function $\mathcal{L}(\hat{y},y)$, the primal non-convex training problem for a two-layer quadratic network where all hidden neurons are connected to all outputs and $w^{j}$ are normalized to unit norm is \cite{BartanPilanci2021}
\begin{equation}\label{nonconvextraining}
\begin{array}{l}
\min\limits_{w^{j},\alpha_j}  \mathcal{L}(\hat{y},y) + \beta\sum_{i=1}^M\|\alpha_i\|_1 \\
\mbox {s.t.}~\hat y^k = \sum_{j=1}^{M}\sigma\left(x^Tw^{j}\right)\alpha_j^k,\\
~~~~\|w^{j}\|_2=1,~k=1,\ldots,p,~j=1,\ldots,M,\\
\end{array}
\end{equation} 
for fixed $a\neq 0, b, c,$ and a fixed regularization coefficient $\beta\ge 0$.
The following result proved in reference \cite{BartanPilanci2021} recasts the training as an equivalent convex optimization problem.
%\vspace{5pt}
\begin{lemma}\cite{BartanPilanci2021}
Given fixed $a\neq 0, b, c,$ and a fixed regularization coefficient $\beta\ge 0$, the solution of the convex problem that is dual to (\ref{nonconvextraining}) and is formulated as
\begin{equation}\label{eqn_QNNConvex}
\begin{array}{l}
\min\limits_{}   \mathcal{L}(\hat{y},y) + \beta\sum_{k=1}^p\left(Z^{k,4}_+ + Z^{k,4}_-\right) \\
\mbox {s.t.}~\hat y_i^k =\\
=\bar x_i^T\left[
\begin{array}{cc}
a\left(Z^{k,1}_+-Z^{k,1}_-\right) & \frac{b}{2}\left(Z^{k,2}_+-Z^{k,2}_-\right)\\
\frac{b}{2}\left(Z^{k,2}_+-Z^{k,2}_-\right)^T & c{\rm \bf Trace}\left(Z^{k,1}_+-Z^{k,1}_-\right)
\end{array}
\right]
\bar x_i\\
Z^{k,4}_+ = {\rm \bf Trace}\left(Z^{k,1}_+\right), Z^{k,4}_- = {\rm \bf Trace}\left(Z^{k,1}_-\right)\\
Z^k_+ =\left[\begin{array}{cc}
		Z^{k,1}_+ & Z^{k,2}_+\\
		\left(Z^{k,2}_+\right)^T& Z^{k,4}_+
		\end{array}
		\right],
		~Z^k_-=\left[\begin{array}{cc}
		Z^{k,1}_- & Z^{k,2}_-\\
	        \left(Z^{k,2}_-\right)^T & Z^{k,4}_-
		\end{array}
		\right]\\
Z^k_+\ge 0,~Z^k_-\ge 0,~\bar x_i^T=[x_i^T~~1]
\end{array}
\end{equation}
for $k=1,\ldots,p,~i=1,\ldots,N$, where $\mathcal L(\cdot)$ is a convex loss function, provides a global optimal solution for the parameters $Z^k_+,Z^k_-\in\RE^{(n+1)\times(n+1)},$ when $M\ge M_*$ with
\begin{equation}\label{neuronsqnn}
M_* = \sum_{k=1}^p\left[rank\left(Z^{k*}_+\right) + rank\left(Z^{k*}_-\right)\right],
\end{equation}
where $Z^{k*}_+$ and $Z^{k*}_-$ for $k=1,\ldots,p$, are the solution of the optimization problem (\ref{eqn_QNNConvex}) given $N$ input data vectors $x_i\in\RE^n$ with corresponding labels $y_i\in\RE^{p}$.
Moreover, the optimal values of the solutions of problems (\ref{nonconvextraining}) and (\ref{eqn_QNNConvex}) are the same. \done
\end{lemma}
%\vspace{10pt}
\begin{remark}
In \cite{BartanPilanci2021} a neural decomposition algorithm is presented to compute the neural network weights given the dual variables $Z^k_+,Z^k_-,~k=1,\ldots,p$.
From this algorithm, it can be concluded that each output can be separately computed from the inputs by a dedicated two-layer quadratic neural network.
This is the reason why the number of neurons for each output has been denoted by $M_k$ in (\ref{thenetworkoutput}).
Therefore, without loss of generality, we consider a single output in the remainder of the paper.
\end{remark}
\section{Bounds on Optimal Value of the Training Cost Function}\label{trainingbounds}
To formulate QNN training as a least squares problem in section \ref{leastsquares}, the Linear Matrix Inequality (LMI) constraints of problem (\ref{eqn_QNNConvex}) will first be removed.
This will enable the computation of lower bounds for $\beta\ge0$.
We will later show that the lower bounds are tight for $\beta=0$.
The following lemmas will be required for the proof of the results in Theorem \ref{lem:QNNlowerboundrls} and Theorem \ref{lem:QNNnoLMIconst}.
%\vspace{10pt}

\begin{lemma}
  \label{lem_Z_decomposition}
  For $a>0$, $b\ne0$, and $c>0$, any real matrix $Z$ of the form
  \begin{equation}
    \label{eqn_lem_Z_form}
    Z = \begin{bmatrix}
      aZ_1 & \frac{b}{2}Z_2 \\ 
      \frac{b}{2}(Z_2)^T & c\textbf{Tr}(Z_1)
    \end{bmatrix} \in\mathbb{R}^{(n+1)\times (n+1)}
  \end{equation}
  has a decomposition $Z=Z^+-Z^-$ such that $Z^+\succeq0$ and $Z^-\succeq0$, where
  \begin{equation}
    Z^+ = \begin{bmatrix}
      aZ_1^+ & \frac{b}{2}Z_2^+ \\ 
      \frac{b}{2}(Z_2^+)^T & c\textbf{Tr}(Z_1^+)
    \end{bmatrix}, \quad Z^- = \begin{bmatrix}
      aZ_1^- & \frac{b}{2}Z_2^- \\ 
      \frac{b}{2}(Z_2^-)^T & c\textbf{Tr}(Z_1^-)
    \end{bmatrix}.
  \end{equation}
\end{lemma}
%\vspace{10pt}
\begin{proof}
%{\bf Proof:} 
  Since $Z$ is a real symmetric matrix, it can be written as the difference between two positive semidefinite (PSD) matrices $Z=\tilde{Z}^+-\tilde{Z}^-$ (see \cite{psd_decomposition_2005}, page 119). Let
  \begin{equation*}
    \tilde{Z}^+ = \begin{bmatrix}
      A_1^+ & A_2^+ \\ (A_2^+)^T & A_4^+
    \end{bmatrix}\succeq0, \quad \tilde{Z}^- = \begin{bmatrix}
      A_1^- & A_2^- \\ (A_2^-)^T & A_4^-
    \end{bmatrix}\succeq0
  \end{equation*}
  where $A_1^+$ and $A_1^-$ are $n\times n$ matrices, $A_2^+$ and $A_2^-$ are vectors of length $n$, and $A_4^+$ and $A_4^-$ are scalars. Since $a>0$, $b\ne0$, $c>0$, without loss of generality, we can rewrite $\tilde{Z}^+$ and $\tilde{Z}^-$ as
  \begin{equation*}
    \tilde{Z}^+ = \begin{bmatrix}
      a\tilde{Z}_1^+ & \frac{b}{2}\tilde{Z}_2^+ \\ \frac{b}{2}(\tilde{Z}_2^+)^T & c\tilde{Z}_4^+
    \end{bmatrix}\succeq0, \quad \tilde{Z}^- = \begin{bmatrix}
      a\tilde{Z}_1^- & \frac{b}{2}\tilde{Z}_2^- \\ \frac{b}{2}(\tilde{Z}_2^-)^T & c\tilde{Z}_4^-
    \end{bmatrix}\succeq0
  \end{equation*}
  where
  \begin{equation*}
    \begin{gathered}
      \tilde{Z}_1^+ = a^{-1}A_1^+, \quad \tilde{Z}_2^+ = 2b^{-1}A_2^+, \quad \tilde{Z}_4^+ = c^{-1}A_4^+ \\
      \tilde{Z}_1^- = a^{-1}A_1^-, \quad \tilde{Z}_2^- = 2b^{-1}A_2^-, \quad \tilde{Z}_4^- = c^{-1}A_4^-
    \end{gathered}
  \end{equation*}
  Without loss of generality, for any $K\in\mathbb{R}^{(n+1)\times (n+1)}$, $K\succeq0$, we can write
  \begin{equation*}
    \begin{split}
      Z &= \tilde{Z}^+ - \tilde{Z}^- + K - K \\
        &= (\tilde{Z}^+ + K) - (\tilde{Z}^- + K) \\
        &= Z^+ - Z^-
    \end{split}
  \end{equation*}
  where $Z^+ = \tilde{Z}^+ + K$ and $Z^- = \tilde{Z}^- + K$.
  Notice that if $K\succeq0$, then $Z^+\succeq0$ and $Z^-\succeq0$, as the summation of two PSD matrices is also PSD. Additionally, from the structure of $Z$ in (\ref{eqn_lem_Z_form}), we know that
  \begin{equation}
    \label{eqn_trZ1_Z4_equality}
    \textbf{Tr}(\tilde{Z}_1^+ - \tilde{Z}_1^-) = \tilde{Z}_4^+ - \tilde{Z}_4^- \Longleftrightarrow \textbf{Tr}(\tilde{Z}_1^+) - \tilde{Z}_4^+ = \textbf{Tr}(\tilde{Z}_1^-) - \tilde{Z}_4^-
  \end{equation}
  This will be used later in the proof.
  There are two cases to consider:
%\vspace{10pt}
  \begin{enumerate}
    \item Case 1: $\textbf{Tr}(\tilde{Z}_1^+) - \tilde{Z}_4^+ \geq 0$ \newline
      Define $k=\textbf{Tr}(\tilde{Z}_1^+) - \tilde{Z}_4^+ \geq 0$ and $K$ as
      \begin{equation*}
        K = \begin{bmatrix}
          O_{n,n} & O_{n,1} \\
          O_{1,n} & ck
        \end{bmatrix}
      \end{equation*}
      Since $ck\geq0$ it follows that $K\succeq0$, and therefore $Z^+=\tilde{Z}^++K\succeq0$ and $Z^-=\tilde{Z}^-+K\succeq0$. Computing $Z^+=\tilde{Z}^++K$ yields
      \begin{equation*}
        Z^+ = \begin{bmatrix}
          a\tilde{Z}^+_1 & \frac{b}{2}\tilde{Z}_2^+ \\
          \frac{b}{2}(\tilde{Z}_2^+)^T & c(\tilde{Z}_4^+ + k )
        \end{bmatrix} = \begin{bmatrix}
          a\tilde{Z}^+_1 & \frac{b}{2}\tilde{Z}_2^+ \\
          \frac{b}{2}(\tilde{Z}_2^+)^T & c\textbf{Tr}(\tilde{Z}_1^+)
        \end{bmatrix}
      \end{equation*}
      Computing $Z^-=\tilde{Z}^-+K$ gives
      \begin{equation*}
        Z^- = \begin{bmatrix}
          a\tilde{Z}^-_1 & \frac{b}{2}\tilde{Z}_2^- \\
          \frac{b}{2}(\tilde{Z}_2^-)^T & c(\tilde{Z}_4^- + k )
        \end{bmatrix} = \begin{bmatrix}
          a\tilde{Z}^-_1 & \frac{b}{2}\tilde{Z}_2^- \\
          \frac{b}{2}(\tilde{Z}_2^-)^T & c(\tilde{Z}_4^-  + \textbf{Tr}(\tilde{Z}_1^+) - \tilde{Z}_4^+)
        \end{bmatrix}
      \end{equation*}
      Using (\ref{eqn_trZ1_Z4_equality}) yields
      \begin{equation*}
        Z^- = \begin{bmatrix}
          a\tilde{Z}^-_1 & \frac{b}{2}\tilde{Z}_2^- \\
          \frac{b}{2}(\tilde{Z}_2^-)^T & c(\tilde{Z}_4^-  + \textbf{Tr}(\tilde{Z}_1^-) - \tilde{Z}_4^-)
        \end{bmatrix} = \begin{bmatrix}
          a\tilde{Z}^-_1 & \frac{b}{2}\tilde{Z}_2^- \\
          \frac{b}{2}(\tilde{Z}_2^-)^T & c\textbf{Tr}(\tilde{Z}_1^-)
        \end{bmatrix}
      \end{equation*}
      Therefore, $Z^+\succeq0$ and $Z^-\succeq0$ can be written as
      \begin{equation*}
        Z^+ = \begin{bmatrix}
            aZ^+_1 & \frac{b}{2}Z^+_2 \\
            \frac{b}{2}(Z_2^+)^T & c\textbf{Tr}(Z_1^+)
        \end{bmatrix}, \quad Z^- = \begin{bmatrix}
            aZ^-_1 & \frac{b}{2}Z^-_2 \\
            \frac{b}{2}(Z_2^-)^T & c\textbf{Tr}(Z_1^-)
        \end{bmatrix}
      \end{equation*}
      where $Z_1^+=\tilde{Z}_1^+$, $Z_1^-=\tilde{Z}_1^-$, $Z_2^+=\tilde{Z}_2^+$, and $Z_2^-=\tilde{Z}_2^-$.
%\vspace{10pt}
    \item Case 2: $\textbf{Tr}(\tilde{Z}_1^+) - \tilde{Z}_4^+ < 0$ \newline
      Define $k=\tilde{Z}_4^+ - \textbf{Tr}(\tilde{Z}_1^+) > 0$ and $K$ as
      \begin{equation*}
        K = \begin{bmatrix}
          an^{-1}kI_{n,n} & O_{n,1} \\
          O_{1,n} & 0
        \end{bmatrix}
      \end{equation*}
      where $I_{n,n}$ is the $n\times n$ identity matrix.
      Since $an^{-1}k>0$ it follows that $an^{-1}kI_{n,n}\succ0$, and thus $K\succeq0$, which implies $Z^+=\tilde{Z}^++K\succeq0$ and $Z^-=\tilde{Z}^-+K\succeq0$. Computing $Z^+=\tilde{Z}^++K$ yields
      \begin{equation*}
        Z^+ = \begin{bmatrix}
          a(\tilde{Z}^+_1 + n^{-1}kI_{n,n}) & \frac{b}{2}\tilde{Z}_2^+ \\
          \frac{b}{2}(\tilde{Z}_2^+)^T & c\tilde{Z}_4^+
        \end{bmatrix} = \begin{bmatrix}
          a(\tilde{Z}^+_1 + n^{-1}(\tilde{Z}_4^+ - \textbf{Tr}(\tilde{Z}_1^+))I_{n,n}) & \frac{b}{2}\tilde{Z}_2^+ \\
          \frac{b}{2}(\tilde{Z}_2^+)^T & c\tilde{Z}_4^+
        \end{bmatrix}
      \end{equation*}
     The trace of the upper left sub-matrix is
      \begin{equation*}
        \begin{split}
          \textbf{Tr}\left(\tilde{Z}^+_1 + n^{-1}(\tilde{Z}_4^+ - \textbf{Tr}(\tilde{Z}_1^+))I_{n,n}\right) &= \textbf{Tr}(\tilde{Z}^+_1) + n^{-1}(\tilde{Z}_4^+ - \textbf{Tr}(\tilde{Z}_1^+))\textbf{Tr}(I_{n,n}) \\
          &= \textbf{Tr}(\tilde{Z}^+_1) + \tilde{Z}_4^+ - \textbf{Tr}(\tilde{Z}_1^+) \\
          &= \tilde{Z}_4^+
        \end{split}
      \end{equation*}
      Computing $Z^-=\tilde{Z}^-+K$ yields
      \begin{equation*}
        Z^- = \begin{bmatrix}
          a(\tilde{Z}^-_1+n^{-1}kI_{n,n}) & \frac{b}{2}\tilde{Z}_2^- \\
          \frac{b}{2}(\tilde{Z}_2^-)^T & c(\tilde{Z}_4^-)
        \end{bmatrix} = \begin{bmatrix}
          a(\tilde{Z}^-_1+n^{-1}(\tilde{Z}_4^+ - \textbf{Tr}(\tilde{Z}_1^+))I_{n,n}) & \frac{b}{2}\tilde{Z}_2^- \\
          \frac{b}{2}(\tilde{Z}_2^-)^T & c(\tilde{Z}_4^-)
        \end{bmatrix}
      \end{equation*}
      Using (\ref{eqn_trZ1_Z4_equality}) yields
      \begin{equation*}
        Z^- = \begin{bmatrix}
          a(\tilde{Z}^-_1+n^{-1}(\tilde{Z}_4^- - \textbf{Tr}(\tilde{Z}_1^-))I_{n,n}) & \frac{b}{2}\tilde{Z}_2^- \\
          \frac{b}{2}(\tilde{Z}_2^-)^T & c(\tilde{Z}_4^-)
        \end{bmatrix}
      \end{equation*}
      The trace of the upper left sub-matrix is
      \begin{equation*}
        \begin{split}
          \textbf{Tr}\left(\tilde{Z}^-_1 + n^{-1}(\tilde{Z}_4^- - \textbf{Tr}(\tilde{Z}_1^-))I_{n,n}\right) &= \textbf{Tr}(\tilde{Z}^-_1) + n^{-1}(\tilde{Z}_4^- - \textbf{Tr}(\tilde{Z}_1^-))\textbf{Tr}(I_{n,n}) \\
          &= \textbf{Tr}(\tilde{Z}^-_1) + \tilde{Z}_4^- - \textbf{Tr}(\tilde{Z}_1^-) \\
          &= \tilde{Z}_4^-
        \end{split}
      \end{equation*}
      Therefore, $Z^+\succeq0$ and $Z^-\succeq0$ can be written as
      \begin{equation*}
        Z^+ = \begin{bmatrix}
            aZ^+_1 & \frac{b}{2}Z^+_2 \\
            \frac{b}{2}(Z_2^+)^T & c\textbf{Tr}(Z_1^+)
        \end{bmatrix}, \quad Z^- = \begin{bmatrix}
            aZ^-_1 & \frac{b}{2}Z^-_2 \\
            \frac{b}{2}(Z_2^-)^T & c\textbf{Tr}(Z_1^-)
        \end{bmatrix}
      \end{equation*}
      where 
\begin{eqnarray*}
Z_1^+&=&\tilde{Z}^+_1+n^{-1}(\tilde{Z}_4^+ - \textbf{Tr}(\tilde{Z}_1^+))I_{n,n},\\
Z_1^-&=&\tilde{Z}^-_1+n^{-1}(\tilde{Z}_4^- - \textbf{Tr}(\tilde{Z}_1^-))I_{n,n},\\
Z_2^+&=&\tilde{Z}_2^+,\\
Z_2^-&=&\tilde{Z}_2^-.
\end{eqnarray*}
  \end{enumerate}
  Thus, the matrix $Z$ can be decomposed as
  \begin{equation*}
    Z = Z^+ - Z^- = \begin{bmatrix}
      aZ_1^+ & \frac{b}{2}Z_2^+ \\ 
      \frac{b}{2}(Z_2^+)^T & c\textbf{Tr}(Z_1^+)
    \end{bmatrix} - \begin{bmatrix}
      aZ_1^- & \frac{b}{2}Z_2^- \\ 
      \frac{b}{2}(Z_2^-)^T & c\textbf{Tr}(Z_1^-)
    \end{bmatrix}
  \end{equation*}
  where $Z^+ \succeq 0$ and $Z^-\succeq0$.\done
%\done
\end{proof}
%\vspace{10pt}

\begin{lemma}
  \label{lem_Z_psd_properies}
  For $a>0$, $c>0$, and $b^2-4ac\geq0$, let Z be any real matrix of the form
  \begin{equation}\label{decompositionofZ}
    Z = \begin{bmatrix}
      aZ_1 & \frac{b}{2}Z_2 \\ 
      \frac{b}{2}(Z_2)^T & c\textbf{Tr}(Z_1)
    \end{bmatrix}.
  \end{equation}
  If $Z\succeq0$ then 
  \begin{equation}
    \label{eqn_lem_psd_req}
    Z'=\begin{bmatrix}
      Z_1 & Z_2 \\ 
      (Z_2)^T & \textbf{Tr}(Z_1)
    \end{bmatrix}\succeq0
  \end{equation}
and
    \begin{equation}\label{schurcondition}
        Z_1\textbf{Tr}(Z_1) \succeq Z_2Z_2^T.
    \end{equation}
\end{lemma}
 %\vspace{10pt}
%{\bf Proof:}
\begin{proof} \cite{Rodrigues2023} 
   Using Schur's complement, $Z\succeq0$ is equivalent to
    \begin{equation}
        \label{eqn_lem_schur_conditions}
        \begin{split}
        c\textbf{Tr}(Z_1) &\geq 0 \\
        \left(1 - \textbf{Tr}(Z_1)\textbf{Tr}^{\dag}(Z_1^+)\right)\frac{b}{2}Z_2^T & = 0 \\
        aZ_1 - \frac{b^2}{4c}Z_2\textbf{Tr}^{\dag}(Z_1)Z_2^T & \succeq 0
        \end{split}
    \end{equation}
    where, defining $t=\textbf{Tr}(Z_1)$, the Moore-Penrose pseudo-inverse $\textbf{Tr}^\dag(Z_1)$ is
    \begin{equation*}
        \textbf{Tr}^{\dag}(Z_1)=\begin{cases}
        0 & t=0 \\
        \textbf{Tr}^{-1}(Z_1) & t \ne 0
        \end{cases}.
    \end{equation*}
Since $b^2\ge4ac>0$,the conditions in (\ref{eqn_lem_schur_conditions}) become $Z_2=0$, $Z_1\succeq0$  when $t=0$. These conditions imply that $Z'\succeq0$ and that the inequality (\ref{schurcondition}) is satisfied.
Since $c>0$, for $t\ne0$ the conditions in (\ref{eqn_lem_schur_conditions}) become $\textbf{Tr}(Z_1)>0$,
    \begin{equation*}
        \frac{4ac}{b^2}Z_1\textbf{Tr}(Z_1) \succeq Z_2Z_2^T.
    \end{equation*}
Given that $b^2-4ac\geq0$ and $ac>0$, this inequality implies that (\ref{schurcondition}) is satisfied.
This implies that $Z'\succeq0$ by Schur's complement.
\done
\end{proof}
%\vspace{10pt}

%We are now ready to state the main result to find a lower bound for (\ref{eqn_QNNConvex}).
The results of the next two theorems will be used to prove the main result in Theorem \ref{upperlowerbounds}.

\begin{theorem}
    \label{lem:QNNlowerboundrls}
For $\beta\ge0,~a>0, c>0, b^2-4ac\ge0$, a lower bound on the optimization problem (\ref{eqn_QNNConvex}) is obtained by solving
    \begin{equation}
        \label{eqn_QNNlowerboundrls}
        \begin{split}
        \underset{Z_1, Z_2}{\min}\; & \mathcal{L}(\hat{y},y)+\frac{\beta}{2}\left(\|Z_1\|_*+\|Z_2\|_2\right) \\
        \text{s.t.}\; & \hat{y}_i = \begin{bmatrix} x_i \\ 1 \end{bmatrix}^T \begin{bmatrix} aZ_1 & \frac{b}{2}Z_2 \\ \frac{b}{2}(Z_2)^T & c\textbf{Tr}(Z_1) \end{bmatrix} \begin{bmatrix} x_i \\ 1 \end{bmatrix}, \quad i=1,\dots,N,
        \end{split}
    \end{equation}
where $\|Z_1\|_*=\|\lambda\|_1$ is the nuclear norm of $Z_1$, and $\lambda=[\lambda_1\ldots\lambda_n]^T$ is the vector of eigenvalues of $Z_1$.
This lower bound is tight when $\beta=0$.
\end{theorem}
 %\vspace{10pt}
%{\bf Proof}:
\begin{proof}
Noting that any feasible point $Z_1,Z_2$ of (\ref{eqn_QNNConvex}) leads to a matrix
\begin{equation*}
Z=\begin{bmatrix} aZ_1 & \frac{b}{2}Z_2 \\ \frac{b}{2}Z_2^T & c\textbf{Tr}\left(Z_1\right) \end{bmatrix},
\end{equation*}
and is also a feasible point of (\ref{eqn_QNNlowerboundrls}), we will compare the value of the two
objective functions for such points.
Any matrix $Z_1=Z_1^T$ corresponding to a feasible point of (\ref{eqn_QNNConvex}) can be decomposed as
\begin{equation}\label{eigendecompositionZ}
Z_1=\sum_{i=1}^n\lambda_iv_iv_i^T,
\end{equation}
where $\lambda_i=\lambda_i\left(Z_1\right)$ are real eigenvalues and the eigenvectors $v_i$ form an orthonormal set for $i=1,\ldots,n$.
The equation (\ref{eigendecompositionZ}) can be rewritten as
\begin{equation}\label{decompositionZ1}
Z_1=\left(\sum_{\lambda_i\ge0}\lambda_iv_iv_i^T\right)-\left(-\sum_{\lambda_i<0}\lambda_iv_iv_i^T\right)=Z_1^+-Z_1^-,~Z_1^+\succeq0, Z_1^-\succeq0.
\end{equation}
The most general decomposition of the form (\ref{decompositionZ1}) is thus
\begin{eqnarray}\label{Z1decomposition}
Z_1^+&=&\sum_{\lambda_i\ge0}\lambda_iv_iv_i^T+K=V\Lambda_+V^T+V\bar KV^T,\nonumber\\
Z_1^-&=&-\sum_{\lambda_i<0}\lambda_iv_iv_i^T+K=-V\Lambda_-V^T+V\bar KV^T,
\end{eqnarray}
where $\bar K=\bar K^T$, $V=[v_1\ldots v_n], V^TV=I$ and, without loss of generality,
\begin{eqnarray*}
\Lambda_+&=&{\rm diag}(\lambda_1^+,\ldots,\lambda_p^+, 0,\ldots,0),\\
~\Lambda_-&=&{\rm diag}(0,\ldots,0,\lambda_1^-,\ldots,\lambda_q^-).
\end{eqnarray*}
Note that the main diagonal of the $\Lambda_+,~\Lambda_-$ matrices have the $p$ non-negative and $q$ negative eigenvalues of $Z_1$, respectively, with $n=p+q$.
Since $Z_1^+\succeq0$ and $Z_1^-\succeq0$, we must have
\begin{eqnarray}\label{constraintsK}
\Lambda_++\bar K &\succeq& 0,\nonumber\\
-\Lambda_-+\bar K &\succeq& 0.
\end{eqnarray}
To satisfy the constraints (\ref{constraintsK}) we conclude that the main diagonal elements of $\bar K$ must be non-negative because of the structures of $\Lambda_+$ and $\Lambda_-$.
Therefore, we must have $\textbf{Tr}(\bar K)\ge0$.
Given that
\begin{eqnarray*}
\textbf{Tr}\left(Z_1^+\right)&=&\textbf{Tr}(\Lambda_+)+ \textbf{Tr}(\bar K),\\
\textbf{Tr}\left(Z_1^-\right)&=& -\textbf{Tr}(\Lambda_-)+ \textbf{Tr}(\bar K),
\end{eqnarray*}
then the regularization term of the objective function of (\ref{eqn_QNNConvex}) is
\begin{equation}\label{nuclearnormbound}
\textbf{Tr}(Z^+_1)+ \textbf{Tr}(Z^-_1)=\sum_{i=1}^n|\lambda_i|+2\textbf{Tr}(\bar K)\ge \|Z_1\|_*
\end{equation}
where $\|Z_1\|_*=\|\lambda\|_1$.
The loss term of the objective function is $\mathcal{L}(\hat{y},y)$ and does not depend on the matrix $\bar K$.

On the other hand, from the Schur complement in the form of the inequality (\ref{schurcondition}) in Lemma \ref{lem_Z_psd_properies}, we find that when $a>0, c>0, b^2-4ac\ge0$, any feasible point of (\ref{eqn_QNNConvex}) must satisfy
\begin{eqnarray*}
Z_1^+\textbf{Tr}(Z_1^+) \succeq Z_2^+\left(Z_2^+\right)^T,\nonumber\\
Z_1^-\textbf{Tr}(Z_1^-) \succeq Z_2^-\left(Z_2^-\right)^T.
\end{eqnarray*}
Applying the trace operator to both sides of the inequality and taking square root, noticing that
$\textbf{Tr}(Z_1^+)\ge0,~\textbf{Tr}(Z_1^-)\ge0$, gives
\begin{eqnarray*}
\textbf{Tr}(Z_1^+) \ge \|Z_2^+\|_2,\nonumber\\
\textbf{Tr}(Z_1^-) \ge \|Z_2^-\|_2.
\end{eqnarray*}
Using the fact that $\|Z_2\|_2=\|Z_2^+-Z_2^-\|_2\le\|Z_2^+\|_2+\|Z_2^-\|_2$, yields
\begin{equation}\label{normtwobound}
\textbf{Tr}\left(Z_1^+\right)+\textbf{Tr}\left(Z_1^-\right)\ge\|Z_2\|_2.
\end{equation}
Therefore, from (\ref{nuclearnormbound}) and (\ref{normtwobound}), the regularization term of the objective function of (\ref{eqn_QNNConvex})  can be lower bounded as follows,
\begin{equation}\label{tobeusedinrls}
\frac{1}{2}2\left[\textbf{Tr}\left(Z_1^+\right)+\textbf{Tr}\left(Z_1^-\right)\right]\ge\frac{1}{2}\left(\|Z_1\|_*+\|Z_2\|_2\right).
\end{equation}
The loss term of the objective function $\mathcal{L}(\hat{y},y)$ is the same for (\ref{eqn_QNNConvex}) and (\ref{eqn_QNNlowerboundrls}).
This proves the lower bound statement because the objective function of the optimization (\ref{eqn_QNNlowerboundrls}) is a lower bound on the objective function of (\ref{eqn_QNNConvex}) for any feasible point of (\ref{eqn_QNNConvex}), which is also a feasible point of (\ref{eqn_QNNlowerboundrls}).

For $\beta=0$ the objective functions of the problems (\ref{eqn_QNNConvex}) and (\ref{eqn_QNNlowerboundrls}) are the same because there is no regularization term.
However, the problem (\ref{eqn_QNNlowerboundrls}) has less constraints and thus its optimal value must be less or equal than the one
of problem (\ref{eqn_QNNConvex}) .
Note however that the satisfaction of the constraints of (\ref{eqn_QNNlowerboundrls}) implies that a feasible point of (\ref{eqn_QNNConvex}) can be found with the same value of the objective function.
To prove this, we use Lemma \ref{lem_Z_decomposition} and then Lemma \ref{lem_Z_psd_properies} twice, the first time with $Z$ replaced by $Z^+$ in (\ref{decompositionofZ}) and the second time with $Z$ replaced by $Z^-$.
Doing this, it follows that for $\beta=0$ any optimal solution of the optimization (\ref{eqn_QNNlowerboundrls}) is also an optimal solution of the optimization (\ref{eqn_QNNConvex}) with the same value of the objective function.
Thus, the lower bound is tight for $\beta=0$.
\done
\end{proof}
%\vspace{10pt}

%The following result will be used in section \ref{leastsquares}.
\begin{theorem}
    \label{lem:QNNnoLMIconst}
For $\beta\ge0$, the optimization problem
    \begin{equation}
        \label{eqn_QNN2LMIconst}
        \begin{split}
        \underset{Z_1, Z_2}{\min}\; & \mathcal{L}(\hat{y},y)+\beta\textbf{Tr}\left(Z_1^++Z_1^-\right) \\
        \text{s.t.}\; & \hat{y}_i = \begin{bmatrix} x_i \\ 1 \end{bmatrix}^T \begin{bmatrix} aZ_1 & \frac{b}{2}Z_2 \\ \frac{b}{2}(Z_2)^T & c\textbf{Tr}(Z_1) \end{bmatrix} \begin{bmatrix} x_i \\ 1 \end{bmatrix}, \quad i=1,\dots,N,\\
       &Z_1=Z_1^+-Z_1^-,~Z_1^+\succeq0,~Z_1^-\succeq0
        \end{split}
    \end{equation}
is equivalent to
    \begin{equation}
        \label{eqn_QNNnoLMIconst}
        \begin{split}
        \underset{Z_1, Z_2}{\min}\; & \mathcal{L}(\hat{y},y)+\beta\|Z_1\|_* \\
        \text{s.t.}\; & \hat{y}_i = \begin{bmatrix} x_i \\ 1 \end{bmatrix}^T \begin{bmatrix} aZ_1 & \frac{b}{2}Z_2 \\ \frac{b}{2}(Z_2)^T & c\textbf{Tr}(Z_1) \end{bmatrix} \begin{bmatrix} x_i \\ 1 \end{bmatrix}, \quad i=1,\dots,N
        \end{split}
    \end{equation}
where $\|Z_1\|_*=\|\lambda\|_1$, and $\lambda=[\lambda_1\ldots\lambda_n]^T$ is the vector of eigenvalues of $Z_1$.
Furthermore, the optimal value of either optimization is a lower bound on the optimal value of (\ref{eqn_QNNConvex}).
This lower bound is tight when $\beta=0,~a>0$, $c>0$, $b^2-4ac \geq 0$.
\end{theorem}
 %\vspace{10pt}
%{\bf Proof}:
\begin{proof}
The feasible points of (\ref{eqn_QNN2LMIconst}) are also feasible points of (\ref{eqn_QNNnoLMIconst}) .
The objective function of (\ref{eqn_QNNnoLMIconst}) is a lower bound on the objective function of (\ref{eqn_QNN2LMIconst}) for any feasible point of (\ref{eqn_QNN2LMIconst}) because
\begin{equation*}
\|Z_1\|_*=\|Z_1^+-Z_1^-\|_*\le\|Z_1^+\|_*+\|Z_1^-\|_*=\textbf{Tr}(Z_1^+)+\textbf{Tr}(Z_1^-)
\end{equation*}
given that $Z_1^+\succeq0$ and $Z_1^-\succeq0$, which implies that the one-norm of the vector of eigenvalues of both $Z_1^+$ and $Z_1^-$ is the sum of the eigenvalues (i.e., the trace).
The proof that the lower bound is tight follows from the proof of Theorem \ref{lem:QNNlowerboundrls}.
In particular, it follows from the inequality (\ref{nuclearnormbound}) and the fact that $\bar K$ does not appear in the loss term $\mathcal{L}(\hat{y},y)$.
Therefore, the optimal $\bar K$ must yield $\textbf{Tr}(\bar K)=0$.
This value of $\textbf{Tr}(\bar K)$ makes the inequality (\ref{nuclearnormbound}) become an equality, which then implies that (\ref{nuclearnormbound}) is a tight lower bound.

To prove that the optimal value of either optimization is a lower bound on the optimal value of (\ref{eqn_QNNConvex}), note that the objective function of (\ref{eqn_QNN2LMIconst}) is the same as the one of (\ref{eqn_QNNConvex}).
Additionally, the objective function of (\ref{eqn_QNNnoLMIconst}) is a lower bound to the objective function of (\ref{eqn_QNNConvex}).
Moreover, the constraints of (\ref{eqn_QNN2LMIconst}) and (\ref{eqn_QNNnoLMIconst}) are a subset of the constraints of (\ref{eqn_QNNConvex}).
This proves the statement that the optimal value of either optimization is a lower bound on the optimal value of (\ref{eqn_QNNConvex}).
The case $\beta=0$ is proved in a similar way as in Theorem \ref{lem:QNNlowerboundrls}.\done
\end{proof}
 %\vspace{10pt}

\begin{remark}\label{nuclearnormremark}
Note that the optimization (\ref{eqn_QNNConvex}) will lead to $Z_1\to0$ and $Z_2\to0$ as $\beta\to\infty$ because $Z_1$ is penalized in the cost and $Z_2$ is related to $Z_1$ through the LMI constraints.
One of the main features of the optimization (\ref{eqn_QNNnoLMIconst}), when compared to (\ref{eqn_QNNConvex}),  is that the matrix $Z_2$ is not penalized in the objective function and is not related to the matrix $Z_1$ through any inequality constraints. Therefore, the optimization (\ref{eqn_QNNnoLMIconst}) will lead to a linear model for $\hat y_i$ as a function of $x_i$ when $\beta\to\infty$. This happens because $Z_1\to0$ but $Z_2\neq0$ as $\beta\to\infty$ so that the loss $\mathcal{L}(\hat{y},y)$ can be minimized.
\end{remark}
 %\vspace{10pt}
\begin{remark}\label{nuclearremark}
The expression for the output $\hat y_i$ in (\ref{eqn_QNNnoLMIconst}) is linear with respect to the elements of $Z_1$ and $Z_2$ because $x_i, a, b, c$ are known.
Therefore, one can rewrite the optimization as the following nuclear norm regularization,
    \begin{equation}
        \label{eqn_QNNMatrixCompletion}
        \begin{split}
        \underset{Z_1, Z_2}{\min}\; & \mathcal{L}(\hat{y},y)+\beta\|Z_1\|_* \\
        \text{s.t.}\; & \hat{y}_i = \mathcal A_i\left(Z_1,Z_2\right), \quad i=1,\dots,N 
        \end{split}
    \end{equation}
where the linear operator $\mathcal A_i$ is given by
\begin{equation}
\mathcal A_i\left(Z_1,Z_2\right)=
\begin{bmatrix} x_i \\ 1 \end{bmatrix}^T \begin{bmatrix} aZ_1 & \frac{b}{2}Z_2 \\ \frac{b}{2}(Z_2)^T & c\textbf{Tr}(Z_1) \end{bmatrix} \begin{bmatrix} x_i \\ 1 \end{bmatrix}.
\end{equation}
As $\beta$ becomes large enough the term $\beta\|Z_1\|_*$ dominates, and  the optimization problem becomes related to the primal problem of nuclear norm minimization.
\end{remark}

\section{Least Squares Training of QNNs}\label{leastsquares}
This section is divided into two subsections: training without regularization and training with regularization.

\subsection{Training Without Regularization}\label{trainingwithoutregularization}
Observing that the input-output expression in (\ref{eqn_QNNlowerboundrls}) relating $x_i$ to $\hat y_i$ is linear with respect to the elements of $Z_1$ and $Z_2$, the problem for $\beta=0$ can be rewritten as
\begin{equation}
    \label{eqn_qnn_ls_form}
    \begin{split}
        \underset{\theta}{\text{min}}\;& \mathcal{L}(\hat{y},y)\\
        \text{s.t.}\;&\hat{y} = H\theta
    \end{split}
\end{equation}
where the regressor matrix is given by
\begin{equation}
    \label{eqn_qnn_H_regressor}
    \begin{gathered}
        H = \begin{bmatrix}
            aH_1+cH_2 & bX
        \end{bmatrix}, \\
        H_1 = \begin{bmatrix}
            \textbf{vec}(x_1x_1^T) \\ \vdots \\ \textbf{vec}(x_Nx_N^T)
        \end{bmatrix}, \quad H_2 = \begin{bmatrix}
            U_{N,n} & O_{N,0.5n(n+1)-n}
        \end{bmatrix}, \quad X=\begin{bmatrix}
            x_1^T \\ \vdots \\ x_N^T
        \end{bmatrix}
    \end{gathered}
\end{equation}
with $\textbf{vec}(x_ix_i^T)$ being a row vector containing all of the upper triangular elements of $x_ix_i^T$, i.e.,
\begin{equation}
    \begin{gathered}
        \textbf{vec}(x_ix_i^T) = \begin{bmatrix} diag_1 & diag_2 & \cdots & diag_n \end{bmatrix}\in\mathbb{R}^{1 \times 0.5n(n+1)} \\
        diag_q = \begin{bmatrix} x_{i_{\{1\}}}x_{i_{\{q\}}} & x_{i_{\{2\}}}x_{i_{\{q+1\}}} & \cdots & x_{i_{\{n-q+1\}}}x_{i_{\{n\}}} \end{bmatrix}
    \end{gathered}
\end{equation}
and the weight vector is
\begin{equation}\label{thetavector}
    \begin{gathered}
        \theta = \begin{bmatrix}\theta_{1,1} & 2\theta_{1,2} & \cdots & 2\theta_{1,n} & \theta_2 \end{bmatrix}^T \in\mathbb{R}^{0.5n(n+1)+n} \\
        \theta_{1,i}= \begin{bmatrix}Z_{1_{\{1,i\}}} & Z_{1_{\{2,i+1\}}} & \cdots & Z_{1_{\{n-i+1,n\}}}\end{bmatrix}, \quad \theta_{2}= \begin{bmatrix}Z_{2_{\{1\}}} & Z_{2_{\{2\}}} & \cdots & Z_{2_{\{n\}}}\end{bmatrix}
    \end{gathered}
\end{equation}
When $\mathcal{L}(\hat{y},y)=||\hat{y}-y||_2$, problem (\ref{eqn_qnn_ls_form}) has the optimal solution
\begin{equation}\label{leassquaresqnn}
    \theta^* = (H^TH)^{-1}H^Ty
\end{equation}
and the network output is given by the linear model
\begin{equation}\label{networkoutput}
\hat y =H\theta^*.
\end{equation}
The desired outputs (labels) can be written as
\begin{equation}\label{ytheta}
y=H\theta+e_y,
\end{equation}
where $e_y$ is the error relative to the linear model $H\theta$.
If we define the output error as
\begin{equation}\label{outputerror}
\delta y=y-\hat y=H\delta\theta+e_y,
\end{equation}
where
\begin{equation}\label{errortheta}
\delta\theta=\theta-\theta^*=\theta-(H^TH)^{-1}H^Ty,
\end{equation}
then replacing (\ref{ytheta}) into (\ref{errortheta}) yields
\begin{equation}
\delta\theta=-(H^TH)^{-1}H^Te_y.
\end{equation}
This expression relates the error $e_y$ to the perturbation $\delta\theta$ in the weights
and it is called the sensitivity of the least squares solution for the weights.
The fact that this sensitivity has an analytical expression is one of the advantages of least squares training of the weights.

\subsection{Training With Regularization}\label{trainingwithregularization}
Although it is not possible to find the optimal solution using least squares when $\beta\neq0$, one can find a lower bound on the optimal value of the objective function using regularized least squares, as shown next.
 %\vspace{10pt}
\begin{theorem}\label{upperlowerbounds}
For $H$ given by (\ref{eqn_qnn_H_regressor}) and
\begin{equation}\label{Mgamma}
	M_\gamma=
       \begin{bmatrix}
       I_{n\times n} & 0_{n\times\frac{n}{2}(n-1)} & 0_{n\times n}\\
       0_{\frac{n}{2}(n-1)\times n} & \frac{1}{\sqrt{2}}I_{\frac{n}{2}(n-1)\times\frac{n}{2}(n-1)} & 
0_{\frac{n}{2}(n-1)\times n}\\
       0_{n\times n} & 0_{n\times\frac{n}{2}(n-1)} & \gamma I_{n\times n}
      \end{bmatrix},
\end{equation}
the solution of
\begin{equation}\label{eqn_qnn_ls_lowerbound}
\begin{array}{l}
        \min~\|y-\hat y\|_2^2+k\beta^2\theta^TM_\gamma^TM_\gamma\theta\\
        \mbox{s.t}~~~\hat{y} = H\theta
\end{array}
\end{equation}
for $k\ge0,~\gamma\ge 0$, is given by
\begin{equation}
    \label{eqn_reg_ls_lowerbound}
    \theta_*^{k,\gamma} = (H^TH+k\beta^2 M_\gamma^TM_\gamma)^{-1}H^Ty.
\end{equation}
A lower bound on the optimal value of (\ref{eqn_QNNnoLMIconst}) when $\mathcal L(\hat y,y)=\|y-\hat y\|_2$ is obtained from (\ref{eqn_reg_ls_lowerbound}) for $k=1, \gamma=0$, as
\begin{equation}\label{lowerbound}
l_b=\sqrt{\left\|y-H\theta_*^{1,0}\right\|_2^2+\beta^2\left(\theta_*^{1,0}\right)^TM_0^TM_0\theta_*^{1,0}}
\end{equation}
An upper bound on the optimal value of (\ref{eqn_QNNnoLMIconst}) when $\mathcal L(\hat y,y)=\|y-\hat y\|_2$ is obtained from (\ref{eqn_reg_ls_lowerbound}) for $k=r={\rm rank}\left(Z_1\right),~\gamma=0$, as
\begin{equation}\label{upperbound}
u_b=\sqrt{2\left\|y-H\theta_*^{r,0}\right\|_2^2+2r\beta^2\left(\theta_*^{r,0}\right)^TM_0^TM_0\theta_*^{r,0}}
\end{equation}
where $Z_1$ is the symmetric matrix whose main diagonal is $\theta_{1,1}$ and the upper diagonals are $\theta_{1,2},\ldots,\theta_{1,n}$ as defined in (\ref{thetavector}), after replacing $\theta$ by $\theta_*^{r,0}$.
Finally, a lower bound on the optimal value of (\ref{eqn_QNNConvex}) when $a>0, c>0, b^2-4ac\ge0$, and $\mathcal L(\hat y,y)=\|y-\hat y\|_2$ is obtained from (\ref{eqn_reg_ls_lowerbound}) for $k=\frac{1}{4},~\gamma=1$, as
\begin{equation}\label{lowerdualbound2}
\bar l_b=\sqrt{\left\|y-H\theta_*^{0.25,1}\right\|_2^2+\frac{\beta^2}{4}\left(\theta_*^{0.25,1}\right)^TM_1^TM_1\theta_*^{0.25,1}}
\end{equation}
\end{theorem}
 %\vspace{10pt}
%{\bf Proof:}
\begin{proof}
The following bounds are valid for the 1-norm,
\begin{equation}\label{normboundslambda}
\|Z_1\|_F=\|\lambda\left(Z_1\right)\|_2\le\|\lambda\left(Z_1\right)\|_1\le\sqrt{r}\|\lambda\left(Z_1\right)\|_2=\sqrt{r}\|Z_1\|_F,
\end{equation}
where the Frobenius norm is
\begin{equation}
\|Z_1\|_F=\sqrt{\sum_{i,j=1}^nZ_{1_{\{i,j\}}}^2}=\sqrt{\theta^TM_0^TM_0\theta}
\end{equation}
for $\theta$ defined in (\ref{thetavector}) and $M_0$ defined in (\ref{Mgamma}),
and where $r={\rm rank}\left(Z_1\right)$.
From (\ref{normboundslambda}) and $\hat y=H\theta$, 
a lower bound on the objective function  of problem (\ref{eqn_QNNnoLMIconst}) when $\mathcal L(\hat y,y)=\|y-\hat y\|_2$ is thus
\begin{equation}\label{glow}
g_{{\rm low}}(y,\beta,\theta)=\sqrt{(y-H\theta)^T(y-H\theta)}+\beta\sqrt{\theta^TM_0^TM_0\theta},
\end{equation}
which in turn is lower bounded by (\ref{lowerbound}).
Minimizing (\ref{lowerbound}) is equivalent to minimizing its square and this
leads to the optimization (\ref{eqn_qnn_ls_lowerbound}) for $k=1,~\gamma=0$.
Using a similar reasoning, one can determine the upper bound
\begin{equation}\
g_{{\rm up}}(y,\beta,\theta)=\sqrt{(y-H\theta)^T(y-H\theta)}+\beta\sqrt{r\theta^TM_0^TM_0\theta},
\end{equation}
which in turn is upper bounded by (\ref{upperbound})
using $\|w\|_1\le\sqrt{2}\|w\|_2$, where
\begin{equation*}
w=\begin{bmatrix}
\sqrt{(y-H\theta)^T(y-H\theta)}\\
\sqrt{r\beta^2\theta^TM_0^TM_0\theta}
\end{bmatrix}.
\end{equation*}
The result then follows by recognizing that minimizing (\ref{upperbound}) is equivalent to minimizing half of its square, and this leads to the optimization (\ref{eqn_qnn_ls_lowerbound}) for $k=r,~\gamma=0$.
For the last result, using the inequalities (\ref{tobeusedinrls}) and (\ref{normboundslambda}),
\begin{equation}\label{rls}
\frac{\beta}{2}2\textbf{Tr}\left(Z_1^++Z_1^-\right)\ge\frac{\beta}{2}\left(\|Z_1\|_*+\|Z_2\|_2\right)\ge\sqrt{\frac{\beta^2}{4}\theta^TM_1^TM_1\theta},
\end{equation}
for $\theta$ defined in (\ref{thetavector}).
Thus, replacing the second term of the objective function of (\ref{eqn_QNNlowerboundrls}) by the rightmost term of the inequalities (\ref{rls}),
and writing the constraints of the problem in the form $\hat y=H\theta$, leads to
a lower bound on the optimal value of (\ref{eqn_QNNConvex}) using (\ref{tobeusedinrls}).
This lower bound is in turn lower bounded by (\ref{lowerdualbound2}).
The proof is complete by noticing that minimizing (\ref{lowerdualbound2}) is equivalent to minimizing its square, which leads to the optimization (\ref{eqn_qnn_ls_lowerbound}) for $k=0.25,~\gamma=1$.\done
\end{proof}
%\done
 %\vspace{10pt}
\begin{remark}
The regularization term using $M_0^TM_0$ as the weighting matrix does not penalize the entries of $Z_2$ in the vector $\theta$. Therefore, the same comment of Remark \ref{nuclearnormremark} applies to the solution in this case.
\end{remark}
\begin{remark}
Since $I\succeq M_1^TM_1\succeq M_0^TM_0$, note that a less tight upper bound on the optimal value of (\ref{eqn_QNNnoLMIconst}) is obtained using the standard regularized least squares solution with $\bar\beta=r\beta^2$, i.e.,
\begin{equation}
    \label{eqn_reg_ls_qnn_solution}
    \theta_*^{r,0} = (H^TH+\bar\beta I)^{-1}H^Ty.
\end{equation}
Note however that the entries of $Z_2$ will be penalized in the regularization term of the objective function and thus $Z_2\to0$ when $\beta\to\infty$.
From Theorem \ref{lem:QNNnoLMIconst} and the solution of (\ref{eqn_qnn_ls_form}) we know that for $\beta=0,~a>0$, $c>0$, $b^2-4ac \geq 0$ the expression (\ref{eqn_reg_ls_qnn_solution}) yields the optimal solution.
Therefore, using small values of $\beta$ one can obtain an approximation of the optimal solution of the problem (\ref{eqn_QNNConvex}). It will be shown in the system identification example that  the solution (\ref{eqn_reg_ls_qnn_solution}) when $\bar\beta$ is replaced by $\beta$ also yields an approximation of the optimal solution, although it is neither a provable upper bound nor a lower bound.
%while preventing invertibility issues that may exist when using the approximation (\ref{eqn_reg_ls_lowerbound}) with $k=r$ when $H$ is not full rank, given that the matrix $M_0^TM_0$ is not invertible, whereas the matrix $I$ is invertible.
\end{remark}
%\vspace{10pt}
%\begin{example}
\section{System Identification Example}\label{sysidexample}
The objective of this example is to perform system identification of a flexible robot arm.
The data is available at the website of the Database for the Identification of Systems (DaISy) \cite{robotarmdata}.
The input $u$ is the measured reaction torque of the arm structure on the ground and the output $z$ is the acceleration of the flexible arm.
The input was a periodic sine sweep. An output of $N=1024$ data points was collected.
The training of the QNN was performed for a quadratic activation function
with parameters $a=0.0937, b=0.5, c=0.4688, \beta=0$, and a quadratic norm loss function $\mathcal L$.
Only the first $122$ data points ($12\%$ of all data) were used to train the network to follow an autoregressive model of the form
\begin{equation}
z_{i+1}=f(z_i,u_i)
\end{equation}
where the nonlinear function $f(z_i,u_i)$ is approximated by a quadratic neural network.
This autoregressive model was proposed in reference \cite{RodriguesGivigi2023}.

\subsection{Training Without Regularization}
\begin{figure}[t] 
\centerline{ \resizebox{91mm}{!}{\includegraphics{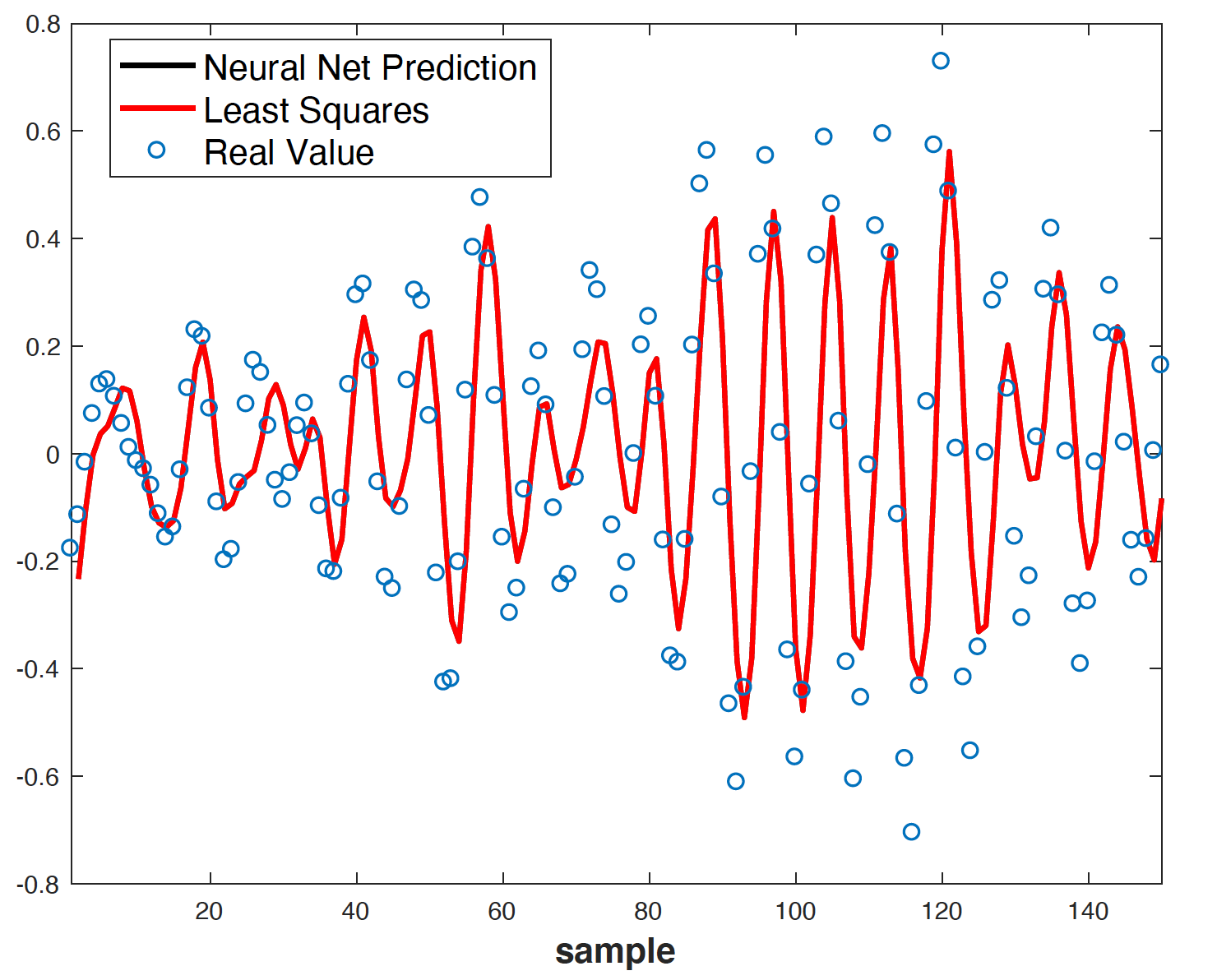}}}
\caption{QNN convex training (black), least square solution (red), and data points (blue circles).
The black and red curves coincide.}
\label{flexiblearm}
\end{figure}
Using the least squares formula (\ref{leassquaresqnn}), the resulting model yields $z_{i+1}=\hat y_i$ where
\begin{eqnarray}\label{eqexample2}
\hat y_i=
\begin{bmatrix}
u_i\\
z_i\\
1
\end{bmatrix}^T
\begin{bmatrix}
aZ_{11} & aZ_{12} & 0.5bZ_{13}\\
aZ_{12} & aZ_{22} & 0.5bZ_{23}\\
0.5bZ_{13} & 0.5bZ_{23} & cZ_{11}+cZ_{22}
\end{bmatrix}
\begin{bmatrix}
u_i\\
z_i\\
1
\end{bmatrix}
\end{eqnarray}
\iffalse
which can be rewritten as
\begin{eqnarray*}
\hat y_i=
H_i
\begin{bmatrix}
Z_{11}\\
Z_{22}\\
2Z_{12}\\
Z_{13}\\
Z_{23}
\end{bmatrix}
\end{eqnarray*}
where $H_i$ is row $i$ of regressor matrix $H$ and is written as
\begin{eqnarray*}
H_i=
\begin{bmatrix}
au_i^2+c & az_i^2+c & au_iz_i & bu_i & bz_i 
\end{bmatrix}
\end{eqnarray*}
The least squares solution (\ref{leassquaresqnn}) yields
\fi
and
\begin{eqnarray*}
Z_{11}=-1.9051\\
Z_{12}=1.4030\\
Z_{13}=-0.8069\\
Z_{22}=1.8830\\
Z_{23}=1.3493
\end{eqnarray*}
Fig. \ref{flexiblearm} compares the data with the predictions of the QNN trained both by solving the convex program (\ref{eqn_QNNConvex}) with $\beta=0$ and a $2$-norm loss function, and by using the least squares formula (\ref{leassquaresqnn}).
As can be seen in the figure, the least squares solution is the same as the solution of the convex program (\ref{eqn_QNNConvex}) but it has the advantage of providing a closed-form expression for the weights. 
%The value of the loss function is $l_{ls}=2.7$.

\subsection{Training with Regularization}
Figure \ref{qnntraining} shows the optimal objective function of problem (\ref{eqn_QNNConvex}) as a function of $\beta$.
Moreover, the lower bound (\ref{lowerdualbound2}) and the regularized least squares solution (\ref{eqn_reg_ls_qnn_solution}) are also computed with $\bar\beta$ replaced by $\beta$.
Although for small values of $\beta$ the regularized least squares solution (\ref{eqn_reg_ls_qnn_solution}) offers a tighter approximation to the optimal solution than the lower bound (\ref{lowerdualbound2}), it is not always a lower bound.
In particular, the optimal value of the objective function is lower than the value for the solution (\ref{eqn_reg_ls_qnn_solution}) when $\beta=10^{-4}$.
By contrast, the solution (\ref{lowerdualbound2}) yields a lower bound for all values of $\beta$.
It is also a tighter approximation than the regularized least squares for $\beta\ge5$.
This suggests that for smaller values of $\beta$ one can use the solution (\ref{eqn_reg_ls_qnn_solution}) for the weights, whereas (\ref{leassquaresqnn}) can be used for larger values of $\beta$ to approximate the optimal solution.
%\end{example}

\section{Conclusions}
This paper proposed a regularized least squares method to train quadratic neural networks (QNNs).
The proposed methodology yields a lower bound on the solution of the training optimization problem for the case where the regularization coefficient is positive.
Moreover, it yields closed-form expressions for the exact optimal solution and its sensitivity when the regularization coefficient is zero.
The proposed approach has three main contributions, namely, (i) it yields an analytical expression for the weights, (ii) an analytical expression for the sensitivity of the weights to errors in the data is also provided, (iii) it establishes a connection between the optimization to compute a lower bound and nuclear norm minimization as per Remark \ref{nuclearremark}.
It was recently shown in our previous work that the training of a deep QNN is equivalent to the training of a two-layer QNN with monomial lifting of the input.
Therefore, to the best of our knowledge, this work provides the first least squares training approach for deep neural networks that does not need to fix or randomize the weights in the hidden layers and can treat them as decision variables.
The proposed least squares training was successfully applied to a nonlinear system identification example where the proposed lower bound is compared with the optimal value.

\section{Acknowledgements}
The authors would like to thank the Natural Sciences and Engineering Research Council (NSERC) of Canada for funding this research.

\begin{figure}[t] 
\centerline{ \resizebox{91mm}{!}{\includegraphics{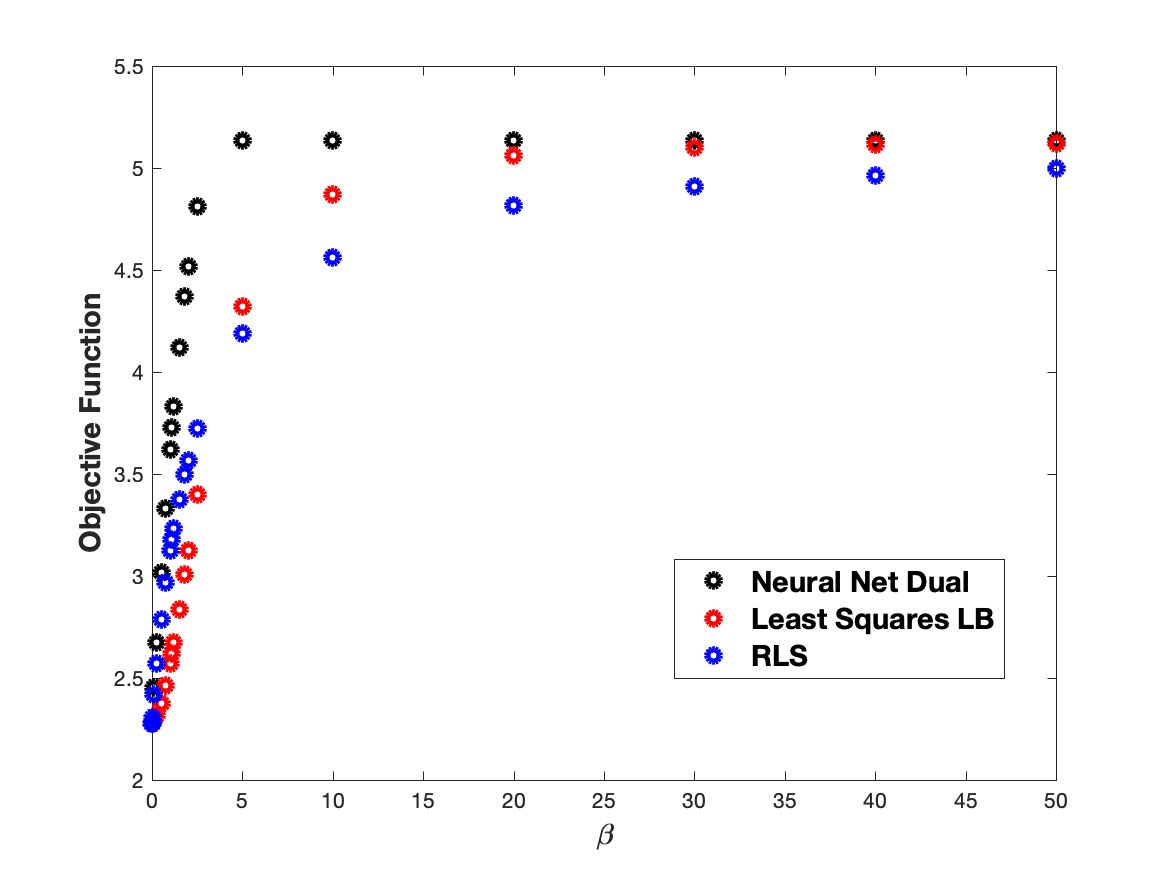}}}
\caption{Optimal values for the training solutions of (\ref{eqn_QNNConvex}), lower bound (\ref{lowerdualbound2}), and regularized least squares (\ref{eqn_reg_ls_qnn_solution}) with $\bar\beta=\beta$ (RLS).}
\label{qnntraining}
\end{figure}
\bibliographystyle{IEEEtran}
\bibliography{nnbibliography}
\end{document}